\documentclass[11pt,a4paper]{article}
\usepackage[hyperref]{emnlp2020}
\usepackage{times}
\usepackage{latexsym}
\usepackage{microtype}
\usepackage{graphicx}
\usepackage{subfigure}
\usepackage{booktabs} %
\usepackage{tikz}
\usepackage{tikzit}
\usepackage{pgfplots}
\usepackage{imakeidx}
\usepackage{url}
\usepackage{mathtools}
\usetikzlibrary{positioning,fit,calc}
\usepackage{amsmath}
\usepackage{forest}
\usepackage{footnote}
\usepackage{amsthm}
\usepackage{algorithmic}
\usepackage{algorithm}
\DeclareMathOperator*{\argmax}{arg\,max}
\newcommand{\m}{m}
\newcommand{\negm}{\Tilde{\m}}
\newcommand{\parent}{\mathsf{parent}}
\newcommand{\leftc}{\mathsf{left}}
\newcommand{\rightc}{\mathsf{right}}
\newcommand{\depth}{D_C}
\newcommand{\tree}{T}
\newcommand{\paramtree}{\tree_C}
\newcommand{\treeset}{\mathcal{\tree}}
\newcommand{\rootv}{\mathsf{root}}
\newcommand{\vrtx}{v}
\newcommand{\leafp}{l}

\newcommand{\srclength}{S}

\newcommand{\trgtok}{x}
\newcommand{\trglength}{N}
\newcommand{\trgstep}{n}
\newcommand{\target}{\trgtok_{1:\trglength}}

\newcommand{\alignment}{L}

\newcommand{\recurse}{\pi}
\newcommand{\succleave}{\mathcal{S}}
\newcommand{\M}{M}
\newcommand{\h}{\Vec{h}}
\newcommand{\produce}{f}
\newcommand{\predleaf}{g_\leafp}
\newcommand{\predwords}{g_\trgtok}

\newcommand{\cell}{\mathsf{Cell}}
\newcommand{\encode}{\mathsf{Encode}}
\newcommand*\circled[1]{\tikz[baseline=(char.base)]{ \node[shape=circle,draw,inner sep=0.8pt] (char) {#1};}}

\usepackage{hyperref}
\newtheorem{proposition}{Proposition} 
\newtheorem{definition}{Definition}

\usepackage{microtype}

\aclfinalcopy %

\title{Recursive Top-Down Production for Sentence Generation \\with Latent Trees}

\author{Shawn Tan\thanks{~~Equal contribution}\\
  Mila / University of Montreal \\
  \texttt{tanjings@mila.quebec} \And
  Yikang Shen$^*$\\
  Mila / University of Montreal \\
  \texttt{yi-kang.shen@umontreal.ca} \AND
Timothy J. O'Donnell \\
  Dept. of Linguistics \\
  Mila / McGill University \\
 Canada CIFAR AI Chair  \And
  Alessandro Sordoni \\
  Microsoft Research \And
  Aaron Courville \\
  Mila / University of Montreal \\
  Canada CIFAR AI Chair}

\date{}

\begin{document}
\maketitle
\begin{abstract}
We model the recursive production property of context-free grammars for natural and synthetic languages.
To this end, we present a dynamic programming algorithm that marginalises over latent binary tree structures with $N$ leaves, allowing us to compute the likelihood of a sequence of $N$ tokens under a latent tree model, which we maximise to train a recursive neural function.
We demonstrate performance on two synthetic tasks: SCAN \citep{lake2017generalization}, where it outperforms previous models on the \textsc{length} split, and English question formation \cite{mccoy2020does}, where it performs comparably to decoders with the ground-truth tree structure.
We also present experimental results on German-English translation on the Multi30k dataset \citep{multi30k}, and qualitatively analyse the induced tree structures our model learns for the SCAN tasks and the German-English translation task.
\end{abstract}


\tikzstyle{word}=[inner sep=1.2pt, thick,fill={rgb,255: red,191; green,191; 
blue,191}, draw=black, shape=circle, tikzit category=tree]
\tikzstyle{blank}=[inner sep=1.2pt, thick, fill=white, draw=black, shape=circle, tikzit category=tree]
\tikzstyle{node_circle}=[draw=red, shape=ellipse, dashed]
\tikzstyle{word_fake}=[fill=black, draw=black, shape=circle]

\tikzstyle{parent}=[-, thick, draw=black]
\tikzstyle{transition}=[->, draw=red]

\section{Introduction}
Given the hierarchical nature of natural language, tree structures have long been considered a fundamental part of natural language understanding. In recent years, a number of studies have shown that  incorporating these structures into deep learning systems can be beneficial for various natural language tasks \cite{socher2013recursive, bowman2015large,eriguchi2016tree}.

Various work has explored the introduction of syntactic structures into recursive encoders, either with explicit syntactic information~\cite{du.w:2020,socher2010learning,dyer2016recurrent} or by means of unsupervised latent tree learning~\cite{williams2018latent,shen2019ordered,kim2019unsupervised}.
Some attempts at formulating structured decoders are~\citet{zhang2015top} and~\citet{alvarez2016tree} which  propose binary top-down tree LSTM architectures for natural language.~\citet{chen2018tree} proposes a tree-structured decoder for code generation.
These methods require ground-truth trees from an external source, and this extra input may not be available for all languages or data sources.

In this work, we propose a tree-based probabilistic decoder model for sequence-to-sequence tasks.
Our model generates sentences from a latent tree structure that aims to reflect natural language syntax. The method assumes that each token in a sentence is emitted at the leaves of a full but latent binary tree (Fig.~\ref{fig:intro}). The tree is obtained by recursively producing node embeddings from a root embedding with a recursive neural network. Word emission probabilities are function of the leaf embeddings. We describe a novel dynamic programming algorithm for exact marginalisation over the large number of latent binary trees.

\begin{figure}[t]
\centering
\includegraphics[scale=0.5]{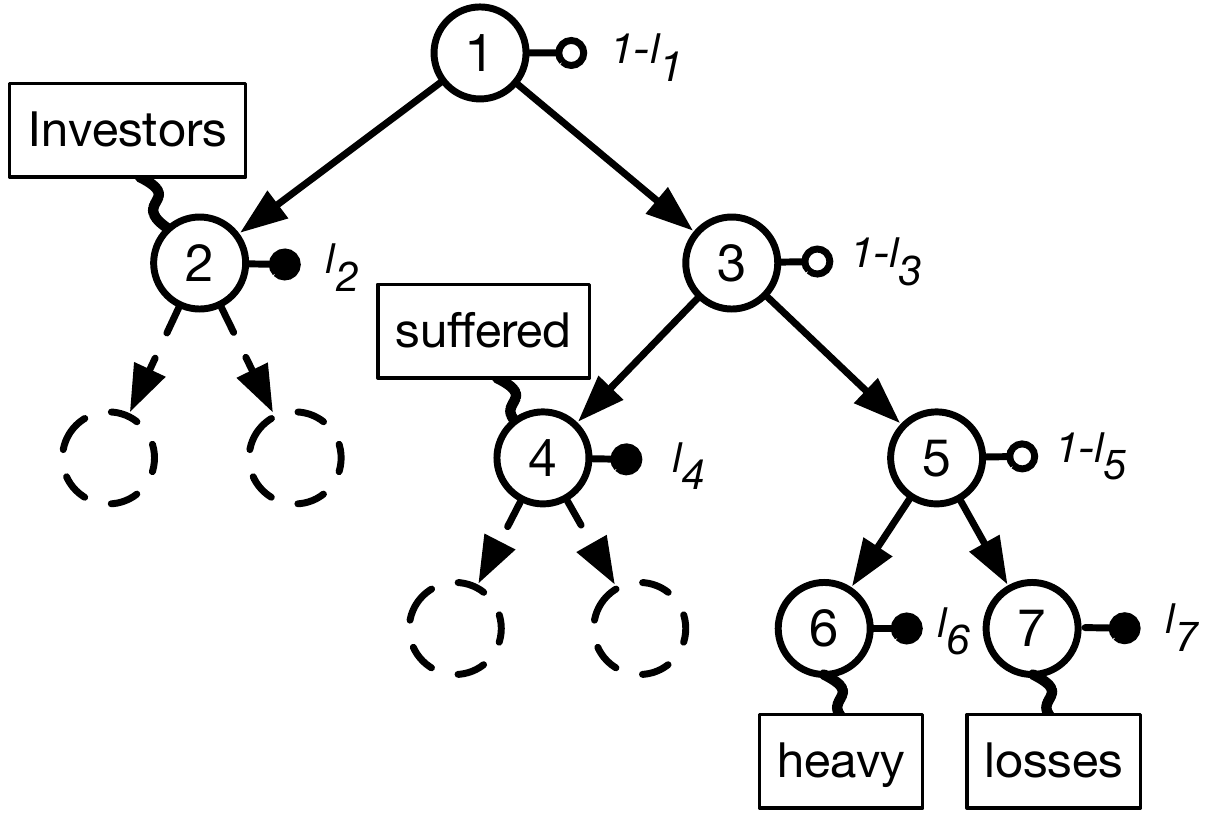}
\caption{Our generative model is a recursive top-down neural network that recursively splits a root node embedding with some node-dependent probability. When the splitting stops, it emits a word with some probability. The joint probability of a sentence and its associated binary tree is the product of the probability of the tree
$(1 - \leafp_1)(1 - \leafp_3) (1-\leafp_5) \leafp_2\leafp_4 \leafp_6 \leafp_7$ and the probabilities of the word emitted at its leaves. We devise a novel marginalisation algorithm over binary trees to compute the likelihood of a sentence.}
\label{fig:intro}
\end{figure}

Our generative model parametrizes a prior over binary trees with a stick-breaking process, similar to the ``penetration probabilities'' defined in~\citet{mochihashi2008infinite}.
It is related to a long tradition of unsupervised grammar induction models that formulate a generative model of sentences~\citep{klein2001natural,bod2006all,klein2005natural}.

Unlike more recent bottom-up approaches such as~\citet{kim2019compound} which require the inside-outside algorithm~\citep{baker1979trainable} to marginalise over tree structures, our approach is top-down and comes with an efficient algorithm to perform marginalisation. Top-down models can be useful, as the decoder is encouraged by design to keep global context while generating sentences~\citep{du2019top,gu-etal-2018-top}. %

In the next section, we will describe the algorithm that marginalises over latent tree structures under some independence assumptions.
We first introduce these assumptions and show that by introducing the notion of successive leaves, we can efficiently sum over different tree structures. We then introduce the details of the recursive architecture used. Finally, we present the experimental results  of the model in Section~\ref{sec:experiments}.

\section{Method}
\subsection{Generative Process}
We assume that each sequence is generated by means of an underlying tree structure which takes the form of a \emph{full binary tree}, which is a tree for which each node is either a leaf or has two children. A sequence of tokens is produced with the following generative process: first, sample a full binary tree $\tree$ from a distribution $p(\tree)$. Denote the sets of leaves of $\tree$ as $\alignment(\tree)$. Then for each leaf $\vrtx$ in $\alignment(\tree)$, sample a token $\trgtok \in \mathcal{V}$, where $\mathcal{V}$ is the vocabulary, from a conditional distribution $p(\trgtok | v)$.

Under this model, the probability of a sequence $\target$ can be obtained by marginalising over possible tree structures with $\trglength$ leaves:
\begin{align}
\begin{split}
p(\target) &= \sum_{\tree} p(\target, \tree) \\
&= \sum_{\tree} p(\target | \tree) p(\tree)
\end{split}\label{eqn:marginal}
\end{align}
We assume that the probability of sequences with lengths different from the number of leaves in the tree is 0.
Our generative process prescribes that, given the tree structure, the probability of each word is independent of the other words,~i.e.:
\begin{align}
    p(\trgtok_{1:\trglength} | \tree)
    &= \prod_{\trgstep=1}^\trglength p(\trgtok_\trgstep\mid \alignment_\trgstep(\tree)), \label{eqn:factorised}
\end{align}
where $L_n(T)$ represents the $n$-th leaf of $\tree$.
In what follows, we describe an algorithm to efficiently marginalise over possible tree structures, such that the involved distributions can be parametrized by neural networks and can be trained end-to-end by maximizing log-likelihood of the observed sequences. We first describe how we model the prior $p(\tree)$ and then how to compute $p(x_{1:N})$ efficiently.

\subsection{Probability of a full binary tree}
\begin{figure*}[t]
\vskip 0.2in
\begin{center}
\input{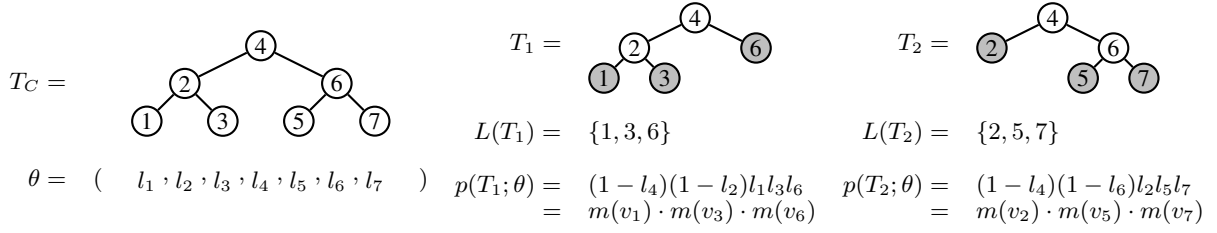}
\caption{In this figure, $\mathsf{root}(\paramtree) = \vrtx_4$. Given $\trglength = 3$, there are two possible trees, $\tree_1$ and $\tree_2$. The probabilities of the trees can be expressed as the recurrent process described in Equation~\ref{eqn:treebreaking}, or as a product of $m(\cdot)$ at the leaf vertices of the internal tree.}
\label{fig:worked_example}
\end{center}
\vskip -0.2in
\end{figure*}
We model the prior probability of a full binary tree $p(\tree)$ by using a branching process similar to the stick-breaking construction, which can be used to model a series of stochastic binary decisions until success \cite{sethuraman1994constructive}. In our model, we perform a series of binary decisions at each vertex, starting at the root and branching downwards. Each decision consists in whether to expand the current node by creating two children or not. This binary decision is therefore modeled with a Bernoulli random variable.

Let us define a complete binary tree $\paramtree$ of depth $\depth$ with vertices $\{\vrtx_1, \dots, \vrtx_M\}, M = {2^{\depth+1} - 1}$.
Each vertex above is associated with a Bernoulli parameter $l$, $\theta = \{\leafp_1, \dots, \leafp_{2^{\depth+1} - 1}\}$, $\leafp_i \in [0, 1]$, modeling its split probability.
The probabilities $(1 - \leafp_i)$ are similar to the ``penetration probabilities'' mentioned in \citet{mochihashi2008infinite}. 
A full binary tree depth $D \le \depth$ is contained in $\paramtree$, so we will refer to it as an \emph{internal tree} from here on\footnote{This is not to be confused with the notion of subtrees.}. See Fig.~1 for an example of two internal trees with three leaves. Its probability can be expressed using parameters $\leafp_i$ as follows.
The probability $p(\tree) = \pi(\mathsf{root})$, where $\pi$ is defined recursively as:
\begin{align}
    \recurse(\vrtx_i) = 
    \left\{\begin{array}{ll}
    \leafp_i      & \textrm{if $\vrtx_i \in \alignment(\tree)$,} \\
    & \\
    (1 - \leafp_i)~\cdot  &  \\
    \quad \recurse(\leftc(\vrtx_i))~\cdot &   \textrm{else}\\
    \qquad \recurse(\rightc(\vrtx_i)) & 
    \end{array}\right.
\label{eqn:treebreaking}
\end{align}
where $\leftc(\vrtx_i)$ and $\rightc(\vrtx_i)$ are the left child and right child respectively.

\subsubsection{Memoizing the value at each vertex}
We can compute Eq.~\eqref{eqn:treebreaking} efficiently by storing a partial computation for each vertex and multiplying the values at the leaves to get the tree probability:
\begin{align}
    p(\tree; \theta) = \prod_{\trgstep=1}^\trglength \m(\alignment_\trgstep(\tree))
    \label{eqn:p_tree}
\end{align}
where $\alignment_\trgstep(\tree)$ denotes the vertex corresponding to the $\trgstep$-th leaf of $T$.
We define this value at the vertex $v_i$ to be $m(v_i)$:
\begin{equation}
\m(v_i) = l_i \prod_{v_j \in V_{i \rightarrow \mathsf{root}}} (1 - l_j)^{\frac{1}{2^{|V_{i \rightarrow j}|}}}
\end{equation}
where $V_{i \rightarrow j}$ denotes the set of vertices in the path from node $v_i$ to node $v_j$ inclusive.
These values can be efficiently computed with this top-down recurrence relation:
\begin{align}
    \m(v_i) &= \left(\negm(\parent(v_i))\right)^{\frac{1}{2}} \cdot l_i \label{eqn:posm} \\
    \negm(v_i) &= \left(\negm(\parent(v_i))\right)^{\frac{1}{2}} \cdot (1 - l_i) \label{eqn:negm} 
\end{align}
where the $\parent(\vrtx_i)$ is the parent of $\vrtx_i$, and $\negm(\parent(\mathsf{root})) = 1$. For example, in Fig.~\ref{fig:worked_example}, $m(1) = (1-\leafp_4)^{1/4} (1-\leafp_2)^{1/2} \leafp_1$ and we demonstrate the case for two internal trees with $D = 2$ and $\trglength = 3$ leaves.

We can then use Eq.~\eqref{eqn:factorised} and Eq.~\eqref{eqn:p_tree} to write the joint probability of a sequence and a tree:
\begin{align}
    p(\trgtok_{1:\trglength}, \tree)
    &= \prod_{\trgstep=1}^\trglength p(\trgtok_\trgstep|\alignment_\trgstep(\tree)) \cdot m(\alignment_\trgstep(\tree)) 
\end{align}
Note that the joint probability factorises as a product over the token probability and the value at the vertex. As we will see later, our method works by traversing the leaves of all possible internal trees, computing the product of the values at the leaves along the way.
Therefore, expressing the probability of a full tree as a product of these values ensures that marginalisation stays tractable.

\subsection{Marginalising over trees}
Now that we can compute the probability of a given tree, we need to marginalise over all full binary trees with exactly $\trglength$ leaves. We will denote this formally by the set $\treeset_\trglength = \{\tree~:~|\alignment(\tree)| = \trglength\}$.
The crux of the problem surrounds marginalising over $\treeset_\trglength$.
We know $|\treeset_\trglength| \leq C_{\trglength-1}$ , where $C_n$ is the $n$-th Catalan number \footnote{\url{https://oeis.org/A000108}}, with equality occuring when $\trglength \leq \depth - 1$.

\paragraph{Successive leaves}
\begin{figure}[t]
\vskip -0.1in
\begin{center}
\small
\begin{tikzpicture}[scale=0.9]
	\begin{pgfonlayer}{nodelayer}
		\node [style=blank] (0) at (-7, -3) {1};
		\node [style=blank] (1) at (-6, -2) {2};
		\node [style=blank] (2) at (-5, -3) {3};
		\node [style=blank] (3) at (-4, -1) {4};
		\node [style=blank] (4) at (-3, -3) {5};
		\node [style=blank] (5) at (-2, -2) {6};
		\node [style=blank] (6) at (-1, -3) {7};
		\node [style=blank] (10) at (-7, -4) {1};
		\node [style=blank] (11) at (-6, -4) {2};
		\node [style=blank] (12) at (-5, -4) {3};
		\node [style=blank] (13) at (-4, -4) {4};
		\node [style=blank] (14) at (-3, -4) {5};
		\node [style=blank] (15) at (-2, -4) {6};
		\node [style=blank] (16) at (-1, -4) {7};
		\node [style=blank] (20) at (-7, -5) {1};
		\node [style=blank] (21) at (-6, -5) {2};
		\node [style=blank] (22) at (-5, -5) {3};
		\node [style=blank] (23) at (-4, -5) {4};
		\node [style=blank] (24) at (-3, -5) {5};
		\node [style=blank] (25) at (-2, -5) {6};
		\node [style=blank] (26) at (-1, -5) {7};
		\node [style=none] (28) at (-8.2, -5) {$\M(\cdot, \trgstep)$};
		\node [style=none] (27) at (-8.2, -4) {$\M(\cdot, {\trgstep-1})$};
	\end{pgfonlayer}
	\begin{pgfonlayer}{edgelayer}
		\draw [style=parent] (3) to (1);
		\draw [style=parent] (1) to (0);
		\draw [style=parent] (1) to (2);
		\draw [style=parent] (3) to (5);
		\draw [style=parent] (5) to (4);
		\draw [style=parent] (5) to (6);
		\draw [style=transition, bend left=15, looseness=0.75] (0) to (2);
		\draw [style=transition, bend left=15, looseness=0.75] (4) to (6);
		\draw [style=transition, bend left=15, looseness=0.75] (1) to (5);
		\draw [style=transition, bend left=15, looseness=0.75] (1) to (4);
		\draw [style=transition, bend left=15, looseness=0.75] (2) to (5);
		\draw [style=transition, bend left=15, looseness=0.75] (2) to (4);
		\draw [style=transition] (10) to (22);
		\draw [style=transition] (14) to (26);
		\draw [style=transition] (11) to (25);
		\draw [style=transition] (11) to (24);
		\draw [style=transition] (12) to (25);
		\draw [style=transition] (12) to (24);
	\end{pgfonlayer}
\end{tikzpicture}
\caption{Successive leaf transitions for a tree of $\depth=2$. The arrows show the possible transitions from each vertex. To enumerate $\treeset_3$ (trees with 3 leaves) we start at any of the vertices the left boundary (\circled{1}, \circled{2}, or \circled{4}), and make 2 transitions (left-to-right arrows) over successive leaves to any vertex in the right boundary (\circled{4}, \circled{6} or \circled{7}), keeping track of the vertices visited along the way.
There are two ways this can be done, which are the examples shown in Figure \ref{fig:worked_example}.}
\label{fig:ctreec}
\end{center}
\end{figure}
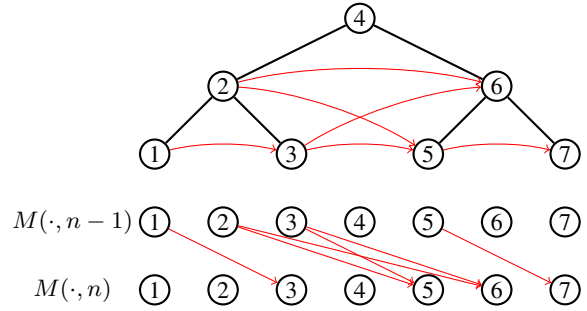
In order to efficiently enumerate all possible internal trees, we define a set of admissible transitions between the vertices of $\paramtree$.
First, let us define the left and right boundaries of a $\paramtree$. Starting from the root node, traversing down the all $\leftc$ children recursively until the leftmost leaf, all vertices visited in this process belong to the left boundary $B_l$.
This notion is similarly defined for all $\rightc$ children in the right boundary $B_r$.
Given a vertex $v$, we define the \emph{successive leaves} of $v$ as any of the next possible leaves in a internal binary tree in which $\vrtx$ is a leaf. 
As an example, in Figure \ref{fig:ctreec}, vertices \circled{5} and \circled{6} are successive leaves of both vertices \circled{2} and \circled{3}. Therefore, if we start at a vertex in the left boundary and travel along these allowed transitions until we reach the right boundary, the vertices visited along this path describe the leaves of an internal tree. This notion is independent of the length of any sequence, and a traversal from the left boundary of $\paramtree$ to the right boundary will induce the leaves of a valid internal $\tree$. As an example, in Figure \ref{fig:ctreec}, the admissible transitions \circled{1} $\rightarrow$ \circled{3} $\rightarrow$ \circled{6} form a valid internal tree, as well as \circled{1} $\rightarrow$ \circled{3} $\rightarrow$ \circled{5} $\rightarrow$ \circled{7}.

To list all pairs of allowed transitions $\vrtx_i$ to $\vrtx_j$, we compute the Cartesian product of the vertices in the right boundary of the left subtree and the left boundary of the right subtree, and do this recursively for each vertex.
See Figure \ref{fig:successive} for an illustration of the concept.
The pseudo-code for generating all such transitions in a tree is shown in Appendix \ref{alg:succ_leaves}: \textsc{SuccessiveLeaves}.
The result of $\textsc{SuccessiveLeaves}(\mathsf{root})$ is the set $\succleave$, which contains pairs of vertices $(\vrtx_i, \vrtx_j)$ such that $\vrtx_j$ is a successive leaf of $\vrtx_i$.
Taking $\trglength - 1$ transitions from the left boundary to the right boundary of $\paramtree$ results in visiting the $\trglength$ leaves of an internal tree.
Proof is in Appendix \ref{sec:proof}.

\begin{figure}[t]
\vskip 0.2in
\begin{center}
\vspace{-2em}
\begin{tikzpicture}
	\begin{pgfonlayer}{nodelayer}
		\node [style=blank] (6) at (0, 0) {};
		\node [style=none] (0) at (-1.5, -1) {};
		\node [style=none] (1) at (-2.5, -3) {};
		\node [style=none] (2) at (-0.5, -3) {};
		\node [style=none] (3) at (1.5, -1) {};
		\node [style=none] (4) at (0.5, -3) {};
		\node [style=none] (5) at (2.5, -3) {};

		\node [style=node_circle, fit={(0) (2)}, inner sep=-6pt, rotate=30] (7)  {};
		\node [style=node_circle, fit={(3) (4)}, inner sep=-6pt, rotate=330] (8)  {};
	\end{pgfonlayer}
	\begin{pgfonlayer}{edgelayer}
		\draw (0.center) to (1.center);
		\draw (2.center) to (1.center);
		\draw (2.center) to (0.center);
		\draw (3.center) to (4.center);
		\draw (5.center) to (4.center);
		\draw (5.center) to (3.center);
		\draw [style=parent] (6) to (0.center);
		\draw [style=parent] (6) to (3.center);
		\draw [style=transition, bend left=15, looseness=0.75] (7) to (8);
	\end{pgfonlayer}
\end{tikzpicture}
\caption{In a binary tree, the left boundary of any right subtree are all successive leaves of the right boundary of its corresponding left subtree.}
\label{fig:successive}
\end{center}
\vskip -0.2in
\end{figure}
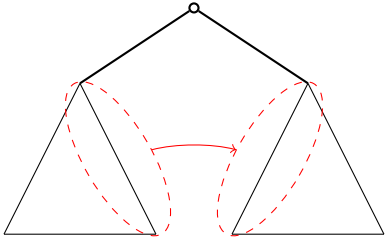

\paragraph{Marginalisation}
We can use our transitions $\mathcal{S}$ to marginalise over internal trees with $N$ leaves as follows: we fill a table $\M(\vrtx, \trgstep)$ that contains the marginal probability of prefix $\trgtok_{1:\trgstep}$, where we sum over all partial trees for which vertex $\vrtx$ has emitted token $\trgtok_\trgstep$:
\begin{align}
\begin{split}
    &\M(\vrtx_i, \trgstep) = \\ 
    &\quad \sum_{\tree~:~\alignment_\trgstep(\tree) = \vrtx_i}\prod_{\trgstep' \leq \trgstep} p(\trgtok_{\trgstep'}|\alignment_{\trgstep'}(\tree)) \cdot m(\alignment_{\trgstep'}(\tree))
\end{split}
\label{eqn:explicit_recurrence}
\end{align}
We first initialise the values at $\M(v, 1)$ at the left boundary:
\begin{align}
    \M(\vrtx_i, 1) &=  \left\{ \begin{array}{ll}
        p(\trgtok_1|\vrtx_i) \cdot m(v_i)  &  \textrm{if $v_i \in B_l$}\\
        0      &  \textrm{else}
    \end{array}\right. \notag
\end{align}
which should be the state of the table for all prefixes sequences of length 1.
Then for $1 < \trgstep \leq \trglength$, 
\begin{align}
    \M(v_i, \trgstep) &= p(\trgtok_\trgstep|\vrtx_i) \cdot m(v_i) ~\smashoperator{\sum_{v_j ~:~ (v_j, v_i) \in \succleave}}~  \M(v_j, \trgstep-1) \label{eqn:recurrence}
\end{align}
where we see that Eq.~\eqref{eqn:explicit_recurrence} can be recovered by pushing the product $p(\trgtok_\trgstep | \vrtx_i)\cdot m(\vrtx_i)$ inside the sum in Eq.~\eqref{eqn:recurrence}. %
The sum describes the situation when vertices have more than one incoming arrow, as depicted in Fig.~\ref{fig:ctreec}.
It should be noted that a large number of these values will be zero, which signify that there are no incomplete trees that end on that vertex.
In order to compute the marginalisation over $\treeset_\trglength$, we have to finally sum over the values at the right boundary:
\begin{align}
    p(\target) = \sum_{\vrtx_i \in B_r} \M(\vrtx_i, \trglength)
\end{align}
since valid full binary trees must also end on the right boundary of $\paramtree$\footnote{Since for any full binary tree, every node has either 0 or 2 children, this means that any full binary tree needs to have one leaf in $B_r$.}.  Note that the values of any trajectory that do not form a full binary tree by $\trglength - 1$ iterations,~i.e. those that do not reach the right boundary, do not get summed.
Another interesting property is that full binary trees with fewer leaves than $\trglength$ would have their trajectories reach the right boundaries much earlier, and those values do not get propagated forward once they do.

\subsection{Decoding from the model}
During decoding, we can perform the following maximisation based on a modification of the marginalisation algorithm,
\begin{align}
    \argmax_{\target, \tree} p(\target, \tree).
\end{align}
This technique borrows heavily from \citet{viterbi1967error}.
We perform the same dynamic programming procedure as above, but replacing summations with maximizations, and maintaining a back-pointer to the summand that was the highest:
\begin{align}
\M^*(v_i, \trgstep) = p(\trgtok_\trgstep|\vrtx_i) &\cdot m(v_i) \\ &\cdot ~\smashoperator{\max_{(\vrtx_j, \vrtx_i) \in \succleave}}~  \M^*(\vrtx_j, \trgstep-1) \notag \label{eqn:max_recurrence}
\end{align}
Since we do not know the length of the sequence being decoded, we need to decide on a stopping criteria. We know that any subsequent multiplication to values in $\M(\cdot, \cdot)$ would decrease it, since $p(\trgtok_\trgstep|\vrtx_i) \cdot m(v_i) \leq 1$.
Thus, we also know that if the current best full sequence has probability $p^*$, then if all probabilities at the frontier are $< p^*$, no sequence with a higher probability can be found.
We can then stop the search, and return the current best solution.
Algorithm \ref{alg:decodejoint} in the Appendix \ref{sec:decode} contains the pseudo-code for decoding.

\section{Architecture}
\subsection{Connectionist Tree (CTree) Decoder}
\begin{figure}[t]
\begin{center}
\includegraphics[width=0.6\linewidth]{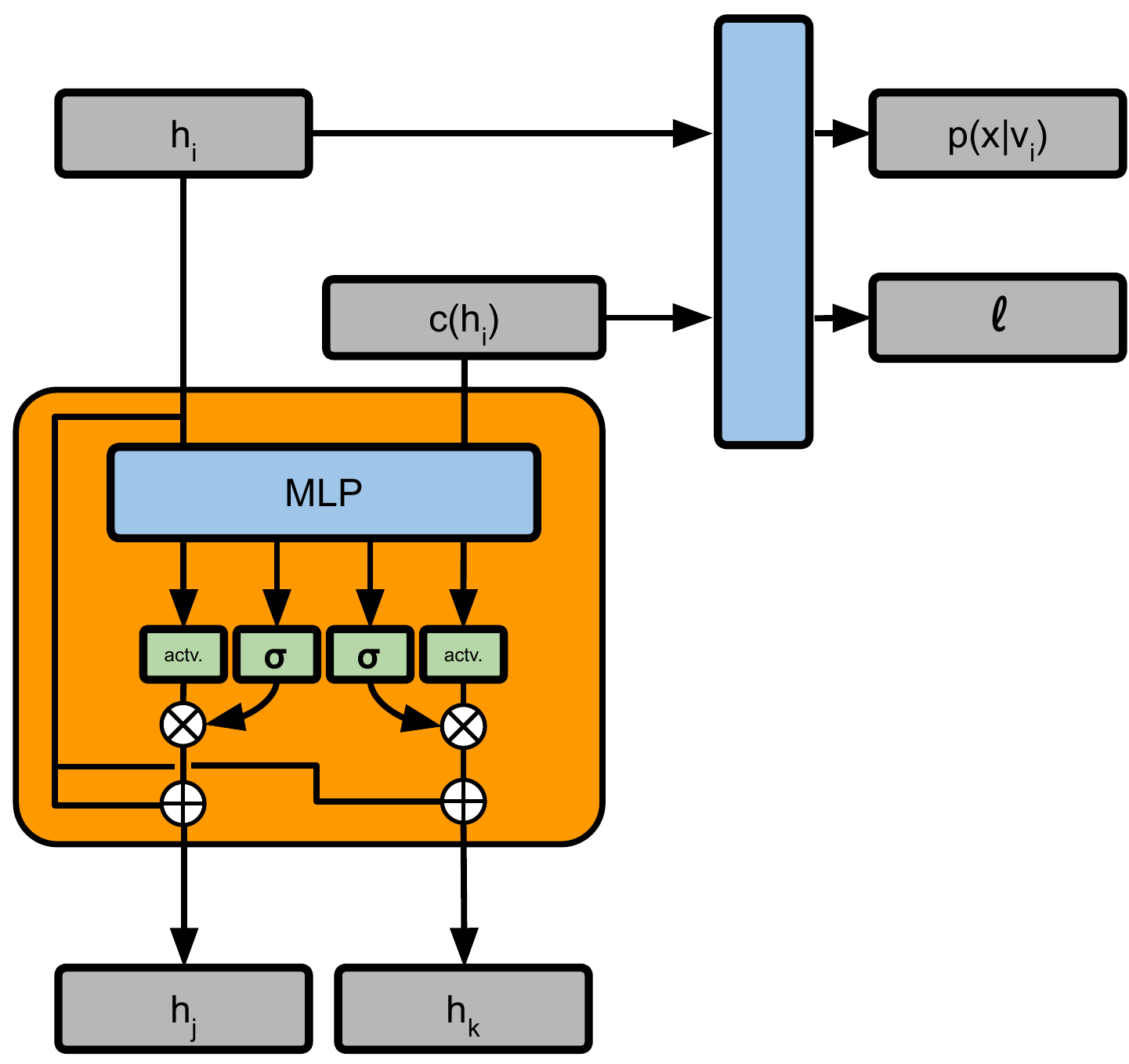}
\caption{Schema of a single production function application.
 From the representation $\h_i$, (1) compute the context vector $c(\h_{i})$ by attending on the encoder, (2) the distribution over word probabilities and the leaf probability parameter $l$, are computed, (3) apply the $\cell(\cdot, \cdot)$ function to produce the child representations $h_j$ and $h_k$. Repeat until the maximum depth is reached.}
\label{fig:model_schematic}
\end{center}
\end{figure}

We parameterize the emission probabilities $p(x | v_i)$ and the splitting probability at each vertex $l_i$ with a recursive neural network.
The neural network recursively splits a root embedding into internal hidden states of the binary tree structures via a \emph{production function} $f$:
\begin{align}
    (\h_{\leftc(\vrtx)}, \h_{\rightc(\vrtx)}) &= \produce(\h_\vrtx, \vec{c}_\vrtx)
\end{align}
where $\h_\vrtx$ is the embedding of the vertex $\vrtx$ and $\vec{c}$ is a generic context embedding that can be optionally vertex dependent and carries external information,~e.g. it can be used to pass information in an encoder-decoder setting.

We parameterise $f(\vec{h}_\vrtx, \vec{c})$ as a gated two layer neural network with a ReLU hidden layer:
\begin{align}
\vec{h} &= \mathsf{relu}(W_1\vec{h}_\vrtx + U_1\vec{c} + b_1) \notag \\
[\vec{c}_\leftc;\vec{c}_\rightc] &= \mathsf{tanh}(\mathsf{layernorm}(W_2 \cdot \vec{h} + \vec{b}_2)) \notag \\
[\vec{g}_\leftc;\vec{g}_\rightc] &= \mathsf{sigmoid}(W_3 \cdot \vec{h} + \vec{b}_3) \notag \\
\vec{h}_\leftc &= \vec{g}_\leftc \odot  \vec{c}_\leftc + (1 - \vec{g}_\leftc) \odot \vec{h}_\vrtx \notag \\
\vec{h}_\rightc &= \vec{g}_\rightc \odot  \vec{c}_\rightc + (1 - \vec{g}_\rightc) \odot \vec{h}_\vrtx \notag
\end{align}
where $\mathsf{layernorm}$ is layer normalization~\cite{ba2016layer}. We fix the hidden size to be two times of the dimension of the input vertex embedding.

The splitting probability $l_\vrtx$ and the emission probabilities $p(x | \vrtx)$ are defined as functions of the vertex embedding:
\begin{align}
p(x | \vrtx) = \predwords(\h_{\vrtx});~~ l_\vrtx = \predleaf(\h_{\vrtx})
\end{align}

The leaf prediction $g_l$ is a linear transform into a two-dimensional output space followed by a softmax. The specific form of the emission probability function $g_x$ can vary with the task. Unless specified, $g_x$ is an MLP.

\subsection{Procedural Description}
Starting with the root representation $\h_{\rho}$ and its eventual contextual information $\vec{c}_{\rho}$, we recursively apply $\produce$.
This can be done efficiently in parallel breadth-wise, doubling the hidden representations at every level.
We apply $\predleaf$ at each level, and then Eq.~\eqref{eqn:posm} and Eq.~\eqref{eqn:negm} to get $m(\vrtx)$, which depend only on the parents. We then apply $f$ recursively until a pre-defined depth $\depth$. We transform all the vertex embeddings using the emission function $\predwords$ in parallel, and multiply $p(\trgtok \mid \vrtx) \cdot m(\vrtx)$ for all vertices and words in the vocabulary. We have now computed the sufficient statistics in order to apply the algorithm described in the previous section to compute the marginal probability of the observed sentence.

$\depth$ is a hyper-parameter that depends on memory and time constraints: if $\depth$ is large, the number of representations grows exponentially with it, as does the time for computing the likelihood. If the depth of the latent trees used to generate the data has an upper bound, we can also restrict the class of trees being learned by setting $\depth$ as well.

\section{Related Work}
Non-parametric Bayesian approaches to learning a hierarchy over the observed data has been proposed in the past~\citep{ghahramani2010tree,griffiths2004hierarchical}. 
These works generally learn a prior on tree-structured data, and assumes a common super-structure that generated the corpus instead of assuming that each observed datapoint may have been produced by a different hierarchical structure. Our generative assumptions are generally stronger but they allow us for tractable marginalisation without costly iterative inference procedures,~e.g. MCMC.

Our method shares similarities with the forward algorithm~\cite{baum1967inequality,baum1968growth} which computes likelihoods for Hidden Markov Models (HMM), and CTC~\cite{graves2006connectionist}.
While the forward algorithm factors in the transition probabilities, both CTC and our algorithm have placed a conditional independence assumption in the factorisation of the likelihood of the output sequence.
The inside-outside algorithm~\cite{baker1979trainable} is usually employed when it comes to learning parameters for PCFGs.~\citet{kim2019compound} gives a modern treatment to PCFGs by introducing Compound PCFGs.
In this work, the CFG production probabilities are conditioned on a continuous latent variable, and the entire model is trained using amortized variational inference~\cite{kingma2013auto}.
This allows the production rules to be conditioned on a sentence-level random variable, allowing it to model correlations over rules that were not possible with a standard PCFG.
However, all co-dependence between the rules can only be captured through the global latent variable.
In CTC, Compound PCFGs, and our work, the fact that the dynamic programming algorithm is differentiable is exploited to train the model.

While typical language modelling is done with a left-to-right autoregressive structure, there has been recent work that change the conditional factorisation order \cite{cho2019non,yang2019xlnet}, and even learn a good factorisation order \cite{stern2019insertion,gu2019insertion}.
For hierarchical text generation,~\citet{chen2018tree} and~\citet{zhang2015tree} have attempted to model this hierarchy using ground-truth parse trees from a parser.
However, the parser was trained based on parses annotated using rules designed by linguists, which presents two challenges: (1) we may not always have these rules, particularly when it comes to low-resource languages, and (2) it may be possible that the structure required for different tasks are slightly different, enforcing the structure based on a universal parse structure may not be optimal.
\citet{jacob2018learning} attempts to learn a tree structure using discrete $\mathsf{split}$ and $\mathsf{merge}$ with REINFORCE \cite{williams1992simple}.
However, the method is known to have high variance \cite{tucker2017rebar}.

There has also been some work that use sequential models for learning a latent hierarchy.
\citet{chung2016hierarchical} again uses discrete binary sampling units to learn a hierarchy.
\citet{shen2018ordered} enforces an ordering to the hidden state of the LSTM \cite{hochreiter1997long} that allows the hidden representations to be interpreted as a tree structure.
In their follow up work, \citet{shen2019ordered} encodes sequences to a single vector representation, which we use in this work as the encoder.
\forestset{
    nice empty nodes/.style={
        for tree={
            s sep=0.1em, 
            l sep=0.3em,
            inner ysep=0.4em, 
            inner xsep=0.04em,
            l=0,
            calign=midpoint,
            fit=tight,
            parent anchor=south,
            child anchor=north,
            delay={if content={}{
                    inner sep=0pt,
                    edge path={\noexpand\path [\forestoption{edge}] 
                    			(!u.parent anchor) 
                               -- (.south)\forestoption{edge label};}
                }{}}
        },
    },
    deeper/.style={
        for tree={
            s sep=0.1em, 
            l sep=1.5em,
            inner ysep=0.4em, 
            inner xsep=0.04em,
            l=0,
            calign=midpoint,
            fit=tight,
            parent anchor=south,
            child anchor=north,
            delay={if content={}{
                    inner sep=0pt,
                    edge path={\noexpand\path [\forestoption{edge}] 
                    			(!u.parent anchor) 
                               -- (.south)\forestoption{edge label};}
                }{}}
        },
    },
}
\newcommand{\drawtree}[1]{\begin{forest}
shape=coordinate,
deeper
#1
\end{forest}}
\newcommand{\treetablecell}[2]{\scalebox{0.9}{\begin{minipage}{\linewidth}\begin{center}
\begin{tabular}{c}
#1 \\
\begin{minipage}{.75\textwidth}
\begin{center}
\drawtree{#2}
\end{center}
\end{minipage}
\end{tabular}
\end{center}
\end{minipage}}}

\section{Experiments}\label{sec:experiments}
We evaluate our method on three different sequence-to-sequence tasks. Unless otherwise stated, we are using the Ordered Memory (OM)~\citep{shen2019ordered} as our encoder. Further details can be found in Appendix~\ref{sec:encoder}.
\subsection{SCAN} \label{sec:scan}
\begin{table*}[t]
\begin{center}
\begin{small}
\begin{sc}
\newcommand{\histo}[1]{\begin{tikzpicture}\begin{axis}[width=100pt, height=53pt, ybar interval, hide axis, ymax=10 ]%
\addplot coordinates { #1 }; \end{axis}\end{tikzpicture}}
\begin{tabular}{ll rl rl rl rl}
\toprule
Model &  & \multicolumn{2}{c}{Simple}
         & \multicolumn{2}{c}{$+$ Turn Left}
         & \multicolumn{2}{c}{$+$ Jump}
         & \multicolumn{2}{c}{Length} \\
\midrule
\multicolumn{2}{l}{\citet{bastings2018jump}}        & 100  & $\pm$ 0.0 
                                                    & 59.1 & $\pm$ 16.8
                                                    & 12.5 & $\pm$ 6.6
                                                    & 18.1 & $\pm$ 1.1  \\
\multicolumn{2}{l}{\citet{bastings2018jump} - dep}  & 100  & $\pm$ 0.0
                                                    & 90.8 & $\pm$ 3.6
                                                    & 0.7  & $\pm$ 0.4
                                                    & 17.8 & $\pm$ 1.7 \\
\multicolumn{2}{l}{\citet{russin2019compositional} (LA)} & 100  & $\pm$ 0.0
                                                    & 99.9 & $\pm$ 0.16
                                                    & 78.4 & $\pm$ 27.4
                                                    & 15.2 & $\pm$ 0.7 \\
\multicolumn{2}{l}{\citet{li2019compositional} (LA)}     & 99.9 & $\pm$ 0.0
                                                    & 99.7 & $\pm$ 0.4
                                                    & 98.8 & $\pm$ 1.4
                                                    & 20.3 & $\pm$ 1.1 \\ 
OM-Seq $\cell$ + LA  & &
99.8 &$\pm$ 0.0 & %
99.4 &$\pm$ 1.4 & %
3.5 &$\pm$ 8.1 & %
20.9 &$\pm$ 3.1 \\ %
 & & 
\multicolumn{2}{c}{\histo{ (0.00, 0.00) (10.00, 0.00) (20.00, 0.00) (30.00, 0.00) (40.00, 0.00) (50.00, 0.00) (60.00, 0.00) (70.00, 0.00) (80.00, 0.00) (90.00, 9.00) (100.00, 0.) }}  & %
\multicolumn{2}{c}{\histo{ (0.00, 0.00) (10.00, 0.00) (20.00, 0.00) (30.00, 0.00) (40.00, 0.00) (50.00, 0.00) (60.00, 0.00) (70.00, 0.00) (80.00, 0.00) (90.00, 10.00) (100.00, 0.) }}  & %
\multicolumn{2}{c}{\histo{ (0.00, 9.00) (10.00, 0.00) (20.00, 1.00) (30.00, 0.00) (40.00, 0.00) (50.00, 0.00) (60.00, 0.00) (70.00, 0.00) (80.00, 0.00) (90.00, 0.00) (100.00, 0.) }}  & %
\multicolumn{2}{c}{\histo{ (0.00, 0.00) (10.00, 8.00) (20.00, 1.00) (30.00, 1.00) (40.00, 0.00) (50.00, 0.00) (60.00, 0.00) (70.00, 0.00) (80.00, 0.00) (90.00, 0.00) (100.00, 0.) }}  \\ %

\cmidrule(lr){1-10}

BiRNN-CTree + LA & & 
99.9 &$\pm$ 0.0 & %
85.5 &$\pm$ 2.2 & %
56.5 &$\pm$ 15.8 & %
19.8 &$\pm$ 0.0 \\ %
 & & 
\multicolumn{2}{c}{\histo{ (0.00, 0.00) (10.00, 0.00) (20.00, 0.00) (30.00, 0.00) (40.00, 0.00) (50.00, 0.00) (60.00, 0.00) (70.00, 0.00) (80.00, 0.00) (90.00, 10.00) (100.00, 0.) }}  & %
\multicolumn{2}{c}{\histo{ (0.00, 0.00) (10.00, 0.00) (20.00, 0.00) (30.00, 0.00) (40.00, 0.00) (50.00, 0.00) (60.00, 0.00) (70.00, 0.00) (80.00, 10.00) (90.00, 0.00) (100.00, 0.) }}  & %
\multicolumn{2}{c}{\histo{ (0.00, 0.00) (10.00, 0.00) (20.00, 1.00) (30.00, 0.00) (40.00, 3.00) (50.00, 1.00) (60.00, 3.00) (70.00, 1.00) (80.00, 1.00) (90.00, 0.00) (100.00, 0.) }}  & %
\multicolumn{2}{c}{\histo{ (0.00, 0.00) (10.00, 10.00) (20.00, 0.00) (30.00, 0.00) (40.00, 0.00) (50.00, 0.00) (60.00, 0.00) (70.00, 0.00) (80.00, 0.00) (90.00, 0.00) (100.00, 0.) }}  \\ %

OM-CTree & &
99.9 &$\pm$ 0.1 & %
93.0 &$\pm$ 7.5 & %
0.1 &$\pm$ 0.2 & %
40.3 &$\pm$ 22.5 \\ %
 & & 
\multicolumn{2}{c}{\histo{ (0.00, 0.00) (10.00, 0.00) (20.00, 0.00) (30.00, 0.00) (40.00, 0.00) (50.00, 0.00) (60.00, 0.00) (70.00, 0.00) (80.00, 0.00) (90.00, 10.00) (100.00, 0.) }}  & %
\multicolumn{2}{c}{\histo{ (0.00, 0.00) (10.00, 0.00) (20.00, 0.00) (30.00, 0.00) (40.00, 0.00) (50.00, 0.00) (60.00, 0.00) (70.00, 1.00) (80.00, 1.00) (90.00, 8.00) (100.00, 0.) }}  & %
\multicolumn{2}{c}{\histo{ (0.00, 10.00) (10.00, 0.00) (20.00, 0.00) (30.00, 0.00) (40.00, 0.00) (50.00, 0.00) (60.00, 0.00) (70.00, 0.00) (80.00, 0.00) (90.00, 0.00) (100.00, 0.) }}  & %
\multicolumn{2}{c}{\histo{ (0.00, 0.00) (10.00, 1.00) (20.00, 3.00) (30.00, 2.00) (40.00, 1.00) (50.00, 1.00) (60.00, 1.00) (70.00, 0.00) (80.00, 0.00) (90.00, 1.00) (100.00, 0.) }}  \\ %

OM-CTree + LA & &
100.0 & $\pm$ 0.0 & %
100.0 & $\pm$ 0.0 & %
80.1 & $\pm$ 17.3 & %
44.7 & $\pm$ 33.5 \\ %
  & & 
 \multicolumn{2}{c}{\histo{(0.00, 0.00) (10.00, 0.00) (20.00, 0.00) (30.00, 0.00) (40.00, 0.00) (50.00, 0.00) (60.00, 0.00) (70.00, 0.00) (80.00, 0.00) (90.00, 10.00) (100.00, 0.)}} & %
\multicolumn{2}{c}{\histo{(0.00, 0.00) (10.00, 0.00) (20.00, 0.00) (30.00, 0.00) (40.00, 0.00) (50.00, 0.00) (60.00, 0.00) (70.00, 0.00) (80.00, 0.00) (90.00, 10.00) (100.00, 0.)}} & %
\multicolumn{2}{c}{\histo{(0.00, 0.00) (10.00, 0.00) (20.00, 0.00) (30.00, 1.00) (40.00, 0.00) (50.00, 0.00) (60.00, 1.00) (70.00, 2.00) (80.00, 3.00) (90.00, 3.00) (100.00, 0.)}} & %
\multicolumn{2}{c}{\histo{(0.00, 0.00) (10.00, 5.00) (20.00, 1.00) (30.00, 1.00) (40.00, 0.00) (50.00, 0.00) (60.00, 0.00) (70.00, 0.00) (80.00, 1.00) (90.00, 2.00) (100.00, 0.)}} \\ %
\bottomrule
\end{tabular}
\end{sc}
\end{small}
\caption{Results on the different splits on the SCAN dataset. The labels are written in the format \textsc{Encoder}-\textsc{Decoder}.
\textsc{CTree + LA} is our decoder with lexical attention. Mean and standard deviation are over 10 runs.}\label{tab:scan}
\end{center}
\end{table*}

The SCAN dataset \cite{lake2017generalization} consists of a set of navigation commands as well as their corresponding action sequences.
As an example, an input of \texttt{\small jump opposite left and walk thrice} shoud yield \texttt{\small LTURN LTURN JUMP WALK WALK WALK}.
The dataset is designed as a test bed for examining the systematic generalization of neural models.
We follow the experiment settings in \citet{bastings2018jump}, where the different splits test for different properties of generalisation.
We apply our model to the 4 experimentation settings and compare our model with the baselines in the literature (See Table \ref{tab:scan}). 

The \textsc{Simple} split has the same data distribution for both the training set and test set.
The \textsc{Turn Left} split partitions the data so that while  \texttt{\small jump left}, and \texttt{\small turn right} would be examples present in the training set, \texttt{\small turn left} are not, but the model must be able to learn from these examples to produce \texttt{\small LTURN} when it sees \texttt{\small turn left} as input.

\paragraph{Lexical Attention} \label{appendix:LA}
\citet{li2019compositional} and \citet{russin2019compositional} propose a similar parameterization of the token output distribution based on key-value attention: the hidden states of the decoder (queries) attend on the hidden states of the encoder (keys), but only a-contextual word embeddings are used as values.
This allows the model to make one-to-one mappings between input token embeddings and output token embeddings (e.g., \texttt{\small jump} in the input always maps to \texttt{\small JUMP} in the output), resulting in huge improvements in performance on the~\textsc{Jump} split.
We refer to this method as \emph{lexical attention} (LA).

\begin{figure}[t]
\vskip -0.2in
\begin{center}
\treetablecell{walk opposite left after look left twice}{[[[[lturn] [look]] [[lturn] [look]]] [[lturn] [[lturn] [walk]]]]}
\caption{Example of a tree inferred by our model from SCAN.}
\label{fig:scantreespos}
\end{center}
\vskip -0.2in
\end{figure}

\paragraph{Results} We report results in Table 1. Our model performs well on the SCAN splits.
Figure \ref{fig:scantreespos} shows one tree induced from a model trained on \textsc{Simple}.
The resulting parses hint at the model learning to ``reuse'' some lower-level concepts when \texttt{\small twice} appears in the input, for instance. The two most challenging tasks are \textsc{Jump} and \textsc{Length} splits.
In \textsc{Jump}, the input token \texttt{\small jump} only appears alone during training and the model has to learn to use it in different contexts during testing.
Surprisingly, this model fails to generalise in the \textsc{Jump} split, suggesting that the capability of our model to perform well on the \textsc{Jump} split may be dependent on the hierarchical decoding as well as the leaf attention.

The \textsc{Length} split partitions the data so that the distribution of output sequences seen in the training set is much shorter than those seen in the test set.
Interestingly, our model converges to a solution that results in a 19.8\% accuracy in 5 out of the 10 random seeds we use.
In the other runs, the model achieves 25\% or higher, with 2 runs achieving $> 99\%$ accuracy.
The high variance of the model deserves more study, but we suspect in the failure cases, the model does not learn a meaningful concept of \texttt{\small thrice}.
Overall, \textsc{Length} requires some generalisation at the structural level during decoding, and has thus far been the most challenging for current sequential models.
Given the results, we believe our model has made some improvements on this front.

\subsection{English Question Formation}    
\citet{mccoy2020does} proposed linguistic synthetic tasks to test for hierarchical inductive biases in models.
One such task is the formation of English questions:  \texttt{\small the zebra does chuckle} $\to$ \texttt{\small does the zebra chuckle ?}.
It gets challenging when further relative clauses are inserted into the sentence: \texttt{\small your zebras that don't dance do chuckle}.
The heuristic that may work in the first case --- moving the first verb to the front of the sentence --- would fail, since the right output would be \texttt{\small do your zebras that don't dance chuckle ?}.
The task involves having two modes of generation, depending on the final token of the input sentence.
If it ends with \texttt{\small DECL}, the decoder simply has to copy the input.
If it ends with \texttt{\small QUEST}, the decoder has to produce the question.
The authors argue, and provide evidence, that the models that do this task well have syntactic structure.
Like SCAN, a generalisation set is included to test for out-of-distribution examples and only the first-word accuracy is reported for the generalisation set.

\paragraph{Results} Training our model on this task, we achieve comparable results to their models that are given the syntactic structure of the sentence, after considering the results of the sequential models that they used.
The results for this task are reported in Table \ref{table:engqn}.

\begin{table}[t]
\begin{center}
\begin{small}
\begin{sc}
\begin{tabular}{l c c}
\toprule
Model &  Full (Test) & First-word (Gen.) \\
\midrule
\multicolumn{3}{l}{\emph{Structure information given}} \\
Tree-Tree  & 0.96 & 0.99 \\
Seq-Tree  & 0.00 & 0.90 \\
Tree-Seq  & 0.96 & 0.13 \\
\midrule
\multicolumn{3}{l}{\emph{No structure information}} \\
Seq-Seq &  0.88 & 0.03 \\
Seq-CTree$^{\dagger *}$ & 1.00 $\pm$ 0.00 & 0.83 $\pm$ 0.19 \\
OM-CTree$^{\dagger *}$ & 1.00 $\pm$ 0.00 & 0.93 $\pm$ 0.07\\
\bottomrule
\end{tabular}
\end{sc}
\end{small}
\caption{English Question Formation results. Our models are annotated with $\dagger$, and we report mean and standard deviation over 5 runs.
Models that use attention are noted with *.}\label{table:engqn}
\end{center}
\end{table}

\subsection{Multi30k Translation}
The Multi30k English-German translation task \cite{multi30k}, is a corpus of short English-German sentence pairs.
The original dataset includes a picture for each pair, but we have excluded them to focus on the text translation task.
Our baseline models include an LSTM sequence-to-sequence with attention, Transformer \citep{vaswani2017attention}, and a non-autoregressive model LaNMT \citep{Shu2020LaNMT}.
For a fair comparison, we trained all models with negative log-likelihood loss or knowledge distillation \citep{kim2016sequence} if applicable.

\paragraph{Results} As shown in Table \ref{table:multi30k}, our model achieved comparable performance to its autoregressive counterparts, and outperforms the non-autoregressive model. 
However, we did not observe significant performance improvements as a result of the generalisation capabilities shown in the previous experiments. 
This suggests further study is needed to overcome remaining issues before deep learning models can really utilise productivity in language.

\begin{figure}[ht]
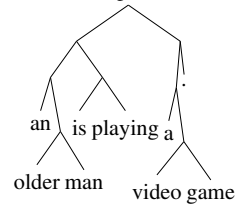

\begin{center}
\scalebox{0.85}{\treetablecell{Ein älterer Mann spielt ein Videospiel.}{[[[[an], [[older] [man]]], [[is] [playing]]], [[[a] [[video] [game]]] [.]]]}}
\caption{Example of a tree inferred by our model from Multi30K De-En.}
\label{fig:translationtreespos}
\end{center}
\end{figure}

On the other hand, examples in Figure \ref{fig:translationtreespos} shows our model does acquire some grammatical knowledge. 
The model tends to generate all noun phrases (e.g. \texttt{\small an older man}, \texttt{\small a video game}) in separate subtrees.
But it also tends to split the sentence before noun phrases. 
For example, the model splits the sub-clause \texttt{\small while in the air} into two different subtrees.
Similarly, previous latent tree induction models \citep{shen2017neural, shen2018ordered} also shows a higher affinity for noun phrases compared to adjective and prepositional phrases.

\begin{table}[t]
\begin{center}
\begin{small}
\begin{sc}
\begin{tabular}{l c c c c}
\toprule
& \multicolumn{2}{c}{En-De} & \multicolumn{2}{c}{De-En} \\
    & Param & Bleu & Param & Bleu \\
\midrule
transformer$^\dagger$ & 69M & 33.6 & 65M & 37.8 \\
LSTM$^\dagger$ & 34M & 35.2 & 30M & 38.0 \\
\midrule
\multicolumn{5}{l}{\emph{Non-autoregressive}} \\
LaNMT$^\ddagger$ & 96M & 26.6 & 96M & 27.9 \\
$+$ distill & 96M & 28.5 & 96M & 32.0 \\
OM-Ctree & 20M & 33.4 & 20M & 34.4 \\
$+$ distill & 20M & 34.7 & 20M & 36.6 \\
\bottomrule
\end{tabular}
\end{sc}
\end{small}
\caption{Multi30K results.
$\dagger$ --- Implemented by OpenNMT \citep{opennmt}. $\ddagger$ --- Trained and finetuned with the released code \url{https://github.com/zomux/lanmt}.}\label{table:multi30k}
\end{center}
\end{table}

\section{Conclusion}
In this paper, we propose a new algorithm for learning a latent structure for sequences of tokens.
Given the current interest in systematic generalisation and compositionality, we hope our work will lead to interesting avenues of research in this direction.

Firstly, the connectionist tree decoding framework allows for different architectural designs for the recurrent function used.
Secondly, while the dynamic programming algorithm is an improvement over a naive enumeration over different trees, there is room for improvement.
For one, exploiting the sparsity of the $\M(\cdot, \cdot)$ table can perhaps result in some memory and time gains.
Finally, the need to recursively expand to a complete tree results in exponential growth with respect to the input length.

These results, while preliminary, suggests that the method holds some potential.
The experimental results reveal some interesting behaviours that require further study.
Nevertheless, we demonstrate that it performs comparably to current algorithms, and surpasses current models in synthetic tasks that have been known to require structure in the models to perform well.
\newpage
\bibliography{comp.bib}

\begin{thebibliography}{52}
\expandafter\ifx\csname natexlab\endcsname\relax\def\natexlab#1{#1}\fi

\bibitem[{Alvarez-Melis and Jaakkola(2016)}]{alvarez2016tree}
David Alvarez-Melis and Tommi~S Jaakkola. 2016.
\newblock Tree-structured decoding with doubly-recurrent neural networks.

\bibitem[{Ba et~al.(2016)Ba, Kiros, and Hinton}]{ba2016layer}
Jimmy~Lei Ba, Jamie~Ryan Kiros, and Geoffrey~E Hinton. 2016.
\newblock Layer normalization.
\newblock \emph{arXiv preprint arXiv:1607.06450}.

\bibitem[{Baker(1979)}]{baker1979trainable}
James~K Baker. 1979.
\newblock Trainable grammars for speech recognition.
\newblock \emph{The Journal of the Acoustical Society of America},
  65(S1):S132--S132.

\bibitem[{Bastings et~al.(2018)Bastings, Baroni, Weston, Cho, and
  Kiela}]{bastings2018jump}
Joost Bastings, Marco Baroni, Jason Weston, Kyunghyun Cho, and Douwe Kiela.
  2018.
\newblock Jump to better conclusions: Scan both left and right.
\newblock In \emph{Proceedings of the 2018 EMNLP Workshop BlackboxNLP:
  Analyzing and Interpreting Neural Networks for NLP}, pages 47--55.

\bibitem[{Baum and Eagon(1967)}]{baum1967inequality}
Leonard~E Baum and John~Alonzo Eagon. 1967.
\newblock An inequality with applications to statistical estimation for
  probabilistic functions of markov processes and to a model for ecology.
\newblock \emph{Bulletin of the American Mathematical Society}, 73(3):360--363.

\bibitem[{Baum and Sell(1968)}]{baum1968growth}
Leonard~E Baum and George Sell. 1968.
\newblock Growth transformations for functions on manifolds.
\newblock \emph{Pacific Journal of Mathematics}, 27(2):211--227.

\bibitem[{Bod(2006)}]{bod2006all}
Rens Bod. 2006.
\newblock An all-subtrees approach to unsupervised parsing.
\newblock In \emph{Proceedings of the 21st International Conference on
  Computational Linguistics and the 44th annual meeting of the Association for
  Computational Linguistics}, pages 865--872. Association for Computational
  Linguistics.

\bibitem[{Bowman et~al.(2015)Bowman, Angeli, Potts, and
  Manning}]{bowman2015large}
Samuel~R Bowman, Gabor Angeli, Christopher Potts, and Christopher~D Manning.
  2015.
\newblock A large annotated corpus for learning natural language inference.
\newblock \emph{arXiv preprint arXiv:1508.05326}.

\bibitem[{Chen et~al.(2018)Chen, Liu, and Song}]{chen2018tree}
Xinyun Chen, Chang Liu, and Dawn Song. 2018.
\newblock Tree-to-tree neural networks for program translation.
\newblock In \emph{Advances in neural information processing systems}, pages
  2547--2557.

\bibitem[{Cho et~al.(2019)Cho, Daum{\'e}~III, Welleck et~al.}]{cho2019non}
Kyunghyun Cho, Hal Daum{\'e}~III, Sean Welleck, et~al. 2019.
\newblock Non-monotonic sequential text generation.
\newblock In \emph{Proceedings of the 2019 Workshop on Widening NLP}, pages
  57--59.

\bibitem[{Cho et~al.(2014)Cho, Van~Merri{\"e}nboer, Bahdanau, and
  Bengio}]{cho2014properties}
Kyunghyun Cho, Bart Van~Merri{\"e}nboer, Dzmitry Bahdanau, and Yoshua Bengio.
  2014.
\newblock On the properties of neural machine translation: Encoder-decoder
  approaches.
\newblock \emph{arXiv preprint arXiv:1409.1259}.

\bibitem[{Chung et~al.(2016)Chung, Ahn, and Bengio}]{chung2016hierarchical}
Junyoung Chung, Sungjin Ahn, and Yoshua Bengio. 2016.
\newblock Hierarchical multiscale recurrent neural networks.
\newblock \emph{arXiv preprint arXiv:1609.01704}.

\bibitem[{Du and Black(2019)}]{du2019top}
Wenchao Du and Alan~W Black. 2019.
\newblock Top-down structurally-constrained neural response generation with
  lexicalized probabilistic context-free grammar.
\newblock In \emph{Proceedings of the 2019 Conference of the North American
  Chapter of the Association for Computational Linguistics: Human Language
  Technologies, Volume 1 (Long and Short Papers)}, pages 3762--3771.

\bibitem[{Du et~al.(2020)Du, Lin, Shen, O'Donnell, Bengio, and
  Zhang}]{du.w:2020}
Wenyu Du, Zhouhan Lin, Yikang Shen, Timothy~J. O'Donnell, Yoshua Bengio, and
  Yue Zhang. 2020.
\newblock Exploiting syntactic structure for better language modeling: A
  syntactic distance approach.
\newblock In \emph{Proceedings of the 58th Annual Meeting of the Association
  for Computational Linguistics}, Seattle, Washington.

\bibitem[{Dyer et~al.(2016)Dyer, Kuncoro, Ballesteros, and
  Smith}]{dyer2016recurrent}
Chris Dyer, Adhiguna Kuncoro, Miguel Ballesteros, and Noah~A Smith. 2016.
\newblock Recurrent neural network grammars.
\newblock In \emph{Proceedings of the 2016 Conference of the North American
  Chapter of the Association for Computational Linguistics: Human Language
  Technologies}, pages 199--209.

\bibitem[{Elliott et~al.(2016)Elliott, Frank, Sima'an, and Specia}]{multi30k}
Desmond Elliott, Stella Frank, Khalil Sima'an, and Lucia Specia. 2016.
\newblock \href {https://doi.org/10.18653/v1/W16-3210} {Multi30k: Multilingual
  {E}nglish-{G}erman image descriptions}.
\newblock In \emph{Proceedings of the 5th Workshop on Vision and Language},
  pages 70--74. Association for Computational Linguistics.

\bibitem[{Eriguchi et~al.(2016)Eriguchi, Hashimoto, and
  Tsuruoka}]{eriguchi2016tree}
Akiko Eriguchi, Kazuma Hashimoto, and Yoshimasa Tsuruoka. 2016.
\newblock Tree-to-sequence attentional neural machine translation.
\newblock In \emph{Proceedings of the 54th Annual Meeting of the Association
  for Computational Linguistics (Volume 1: Long Papers)}, pages 823--833.

\bibitem[{Ghahramani et~al.(2010)Ghahramani, Jordan, and
  Adams}]{ghahramani2010tree}
Zoubin Ghahramani, Michael~I Jordan, and Ryan~P Adams. 2010.
\newblock Tree-structured stick breaking for hierarchical data.
\newblock In \emph{Advances in neural information processing systems}, pages
  19--27.

\bibitem[{Graves et~al.(2006)Graves, Fern{\'a}ndez, Gomez, and
  Schmidhuber}]{graves2006connectionist}
Alex Graves, Santiago Fern{\'a}ndez, Faustino Gomez, and J{\"u}rgen
  Schmidhuber. 2006.
\newblock Connectionist temporal classification: labelling unsegmented sequence
  data with recurrent neural networks.
\newblock In \emph{Proceedings of the 23rd international conference on Machine
  learning}, pages 369--376.

\bibitem[{Griffiths et~al.(2004)Griffiths, Jordan, Tenenbaum, and
  Blei}]{griffiths2004hierarchical}
Thomas~L Griffiths, Michael~I Jordan, Joshua~B Tenenbaum, and David~M Blei.
  2004.
\newblock Hierarchical topic models and the nested chinese restaurant process.
\newblock In \emph{Advances in neural information processing systems}, pages
  17--24.

\bibitem[{G{\=u} et~al.(2018)G{\=u}, Shavarani, and Sarkar}]{gu-etal-2018-top}
Jetic G{\=u}, Hassan~S. Shavarani, and Anoop Sarkar. 2018.
\newblock \href {https://doi.org/10.18653/v1/D18-1037} {Top-down tree
  structured decoding with syntactic connections for neural machine translation
  and parsing}.
\newblock In \emph{Proceedings of the 2018 Conference on Empirical Methods in
  Natural Language Processing}, pages 401--413, Brussels, Belgium. Association
  for Computational Linguistics.

\bibitem[{Gu et~al.(2019)Gu, Liu, and Cho}]{gu2019insertion}
Jiatao Gu, Qi~Liu, and Kyunghyun Cho. 2019.
\newblock Insertion-based decoding with automatically inferred generation
  order.
\newblock \emph{Transactions of the Association for Computational Linguistics},
  7:661--676.

\bibitem[{Hochreiter and Schmidhuber(1997)}]{hochreiter1997long}
Sepp Hochreiter and J{\"u}rgen Schmidhuber. 1997.
\newblock Long short-term memory.
\newblock \emph{Neural computation}, 9(8):1735--1780.

\bibitem[{Jacob et~al.(2018)Jacob, Lin, Sordoni, and
  Bengio}]{jacob2018learning}
Athul~Paul Jacob, Zhouhan Lin, Alessandro Sordoni, and Yoshua Bengio. 2018.
\newblock Learning hierarchical structures on-the-fly with a
  recurrent-recursive model for sequences.
\newblock In \emph{Proceedings of The Third Workshop on Representation Learning
  for NLP}, pages 154--158.

\bibitem[{Kim et~al.(2019{\natexlab{a}})Kim, Dyer, and Rush}]{kim2019compound}
Yoon Kim, Chris Dyer, and Alexander~M Rush. 2019{\natexlab{a}}.
\newblock Compound probabilistic context-free grammars for grammar induction.
\newblock In \emph{Proceedings of the 57th Annual Meeting of the Association
  for Computational Linguistics}, pages 2369--2385.

\bibitem[{Kim and Rush(2016)}]{kim2016sequence}
Yoon Kim and Alexander~M Rush. 2016.
\newblock Sequence-level knowledge distillation.
\newblock \emph{arXiv preprint arXiv:1606.07947}.

\bibitem[{Kim et~al.(2019{\natexlab{b}})Kim, Rush, Yu, Kuncoro, Dyer, and
  Melis}]{kim2019unsupervised}
Yoon Kim, Alexander~M Rush, Lei Yu, Adhiguna Kuncoro, Chris Dyer, and G{\'a}bor
  Melis. 2019{\natexlab{b}}.
\newblock Unsupervised recurrent neural network grammars.
\newblock \emph{arXiv preprint arXiv:1904.03746}.

\bibitem[{Kingma and Welling(2013)}]{kingma2013auto}
Diederik~P Kingma and Max Welling. 2013.
\newblock Auto-encoding variational bayes.
\newblock \emph{arXiv preprint arXiv:1312.6114}.

\bibitem[{Klein and Manning(2001)}]{klein2001natural}
Dan Klein and Christopher~D Manning. 2001.
\newblock Natural language grammar induction using a constituent-context model.
\newblock In \emph{Proceedings of the 14th International Conference on Neural
  Information Processing Systems: Natural and Synthetic}, pages 35--42.

\bibitem[{Klein and Manning(2005)}]{klein2005natural}
Dan Klein and Christopher~D Manning. 2005.
\newblock Natural language grammar induction with a generative
  constituent-context model.
\newblock \emph{Pattern recognition}, 38(9):1407--1419.

\bibitem[{Klein et~al.(2017)Klein, Kim, Deng, Senellart, and Rush}]{opennmt}
Guillaume Klein, Yoon Kim, Yuntian Deng, Jean Senellart, and Alexander~M. Rush.
  2017.
\newblock \href {https://doi.org/10.18653/v1/P17-4012} {Open{NMT}: Open-source
  toolkit for neural machine translation}.
\newblock In \emph{Proc. ACL}.

\bibitem[{Lake and Baroni(2017)}]{lake2017generalization}
Brenden~M Lake and Marco Baroni. 2017.
\newblock Generalization without systematicity: On the compositional skills of
  sequence-to-sequence recurrent networks.
\newblock \emph{arXiv preprint arXiv:1711.00350}.

\bibitem[{Li et~al.(2019)Li, Zhao, Wang, and Hestness}]{li2019compositional}
Yuanpeng Li, Liang Zhao, Jianyu Wang, and Joel Hestness. 2019.
\newblock Compositional generalization for primitive substitutions.
\newblock In \emph{Proceedings of the 2019 Conference on Empirical Methods in
  Natural Language Processing and the 9th International Joint Conference on
  Natural Language Processing (EMNLP-IJCNLP)}, pages 4284--4293.

\bibitem[{McCoy et~al.(2020)McCoy, Frank, and Linzen}]{mccoy2020does}
R~Thomas McCoy, Robert Frank, and Tal Linzen. 2020.
\newblock Does syntax need to grow on trees? sources of hierarchical inductive
  bias in sequence-to-sequence networks.
\newblock \emph{arXiv preprint arXiv:2001.03632}.

\bibitem[{Mochihashi and Sumita(2008)}]{mochihashi2008infinite}
Daichi Mochihashi and Eiichiro Sumita. 2008.
\newblock The infinite markov model.
\newblock In \emph{Advances in neural information processing systems}, pages
  1017--1024.

\bibitem[{Russin et~al.(2019)Russin, Jo, and
  O'Reilly}]{russin2019compositional}
Jake Russin, Jason Jo, and Randall~C O'Reilly. 2019.
\newblock Compositional generalization in a deep seq2seq model by separating
  syntax and semantics.
\newblock \emph{arXiv preprint arXiv:1904.09708}.

\bibitem[{Sethuraman(1994)}]{sethuraman1994constructive}
Jayaram Sethuraman. 1994.
\newblock A constructive definition of dirichlet priors.
\newblock \emph{Statistica sinica}, pages 639--650.

\bibitem[{Shen et~al.(2017)Shen, Lin, Huang, and Courville}]{shen2017neural}
Yikang Shen, Zhouhan Lin, Chin-Wei Huang, and Aaron Courville. 2017.
\newblock Neural language modeling by jointly learning syntax and lexicon.
\newblock \emph{arXiv preprint arXiv:1711.02013}.

\bibitem[{Shen et~al.(2019)Shen, Tan, Hosseini, Lin, Sordoni, and
  Courville}]{shen2019ordered}
Yikang Shen, Shawn Tan, Arian Hosseini, Zhouhan Lin, Alessandro Sordoni, and
  Aaron~C Courville. 2019.
\newblock Ordered memory.
\newblock In \emph{Advances in Neural Information Processing Systems}, pages
  5038--5049.

\bibitem[{Shen et~al.(2018)Shen, Tan, Sordoni, and Courville}]{shen2018ordered}
Yikang Shen, Shawn Tan, Alessandro Sordoni, and Aaron Courville. 2018.
\newblock Ordered neurons: Integrating tree structures into recurrent neural
  networks.
\newblock \emph{arXiv preprint arXiv:1810.09536}.

\bibitem[{Shu et~al.(2020)Shu, Lee, Nakayama, and Cho}]{Shu2020LaNMT}
Raphael Shu, Jason Lee, Hideki Nakayama, and Kyunghyun Cho. 2020.
\newblock Latent-variable non-autoregressive neural machine translation with
  deterministic inference using a delta posterior.
\newblock \emph{AAAI}.

\bibitem[{Socher et~al.(2010)Socher, Manning, and Ng}]{socher2010learning}
Richard Socher, Christopher~D Manning, and Andrew~Y Ng. 2010.
\newblock Learning continuous phrase representations and syntactic parsing with
  recursive neural networks.
\newblock In \emph{Proceedings of the NIPS-2010 Deep Learning and Unsupervised
  Feature Learning Workshop}, volume 2010, pages 1--9.

\bibitem[{Socher et~al.(2013)Socher, Perelygin, Wu, Chuang, Manning, Ng, and
  Potts}]{socher2013recursive}
Richard Socher, Alex Perelygin, Jean Wu, Jason Chuang, Christopher~D Manning,
  Andrew Ng, and Christopher Potts. 2013.
\newblock Recursive deep models for semantic compositionality over a sentiment
  treebank.
\newblock In \emph{Proceedings of the 2013 conference on empirical methods in
  natural language processing}, pages 1631--1642.

\bibitem[{Stern et~al.(2019)Stern, Chan, Kiros, and
  Uszkoreit}]{stern2019insertion}
Mitchell Stern, William Chan, Jamie Kiros, and Jakob Uszkoreit. 2019.
\newblock Insertion transformer: Flexible sequence generation via insertion
  operations.
\newblock In \emph{International Conference on Machine Learning}, pages
  5976--5985.

\bibitem[{Tucker et~al.(2017)Tucker, Mnih, Maddison, Lawson, and
  Sohl-Dickstein}]{tucker2017rebar}
George Tucker, Andriy Mnih, Chris~J Maddison, John Lawson, and Jascha
  Sohl-Dickstein. 2017.
\newblock Rebar: Low-variance, unbiased gradient estimates for discrete latent
  variable models.
\newblock In \emph{Advances in Neural Information Processing Systems}, pages
  2627--2636.

\bibitem[{Vaswani et~al.(2017)Vaswani, Shazeer, Parmar, Uszkoreit, Jones,
  Gomez, Kaiser, and Polosukhin}]{vaswani2017attention}
Ashish Vaswani, Noam Shazeer, Niki Parmar, Jakob Uszkoreit, Llion Jones,
  Aidan~N Gomez, {\L}ukasz Kaiser, and Illia Polosukhin. 2017.
\newblock Attention is all you need.
\newblock In \emph{Advances in Neural Information Processing Systems}, pages
  5998--6008.

\bibitem[{Viterbi(1967)}]{viterbi1967error}
Andrew Viterbi. 1967.
\newblock Error bounds for convolutional codes and an asymptotically optimum
  decoding algorithm.
\newblock \emph{IEEE transactions on Information Theory}, 13(2):260--269.

\bibitem[{Williams et~al.(2018)Williams, Drozdov*, and
  Bowman}]{williams2018latent}
Adina Williams, Andrew Drozdov*, and Samuel~R Bowman. 2018.
\newblock Do latent tree learning models identify meaningful structure in
  sentences?
\newblock \emph{Transactions of the Association of Computational Linguistics},
  6:253--267.

\bibitem[{Williams(1992)}]{williams1992simple}
Ronald~J Williams. 1992.
\newblock Simple statistical gradient-following algorithms for connectionist
  reinforcement learning.
\newblock \emph{Machine learning}, 8(3-4):229--256.

\bibitem[{Yang et~al.(2019)Yang, Dai, Yang, Carbonell, Salakhutdinov, and
  Le}]{yang2019xlnet}
Zhilin Yang, Zihang Dai, Yiming Yang, Jaime Carbonell, Russ~R Salakhutdinov,
  and Quoc~V Le. 2019.
\newblock Xlnet: Generalized autoregressive pretraining for language
  understanding.
\newblock In \emph{Advances in neural information processing systems}, pages
  5754--5764.

\bibitem[{Zhang et~al.(2015{\natexlab{a}})Zhang, Lu, and Lapata}]{zhang2015top}
Xingxing Zhang, Liang Lu, and Mirella Lapata. 2015{\natexlab{a}}.
\newblock Top-down tree long short-term memory networks.
\newblock \emph{arXiv preprint arXiv:1511.00060}.

\bibitem[{Zhang et~al.(2015{\natexlab{b}})Zhang, Lu, and
  Lapata}]{zhang2015tree}
Xingxing Zhang, Liang Lu, and Mirella Lapata. 2015{\natexlab{b}}.
\newblock Tree recurrent neural networks with application to language modeling.
\newblock \emph{CoRR, abs/1511.00060}.

\end{thebibliography}
\bibliographystyle{acl_natbib}
\newpage
\onecolumn
\appendix
\section{Proofs} \label{sec:proof}
In this context, all trees are rooted.

\begin{definition}\label{def:full}
A \textbf{full binary tree} is a tree where each vertex has either 0 or 2 children.
\end{definition}
\begin{definition}
A \textbf{complete binary tree} $\paramtree$ is a tree where each vertex that is not a leaf has 2 children.
\end{definition}
\begin{definition}\label{def:internal}
An \textbf{internal tree} $\tree$ of a complete binary tree $\paramtree$ is a full binary tree $\tree$ such that $\rootv(\tree) = \rootv(\paramtree)$ and whose vertices and edges are a subset of $\paramtree$.
\end{definition}
\begin{definition}
The set $\treeset(\paramtree)$ of all internal trees of $\paramtree$.
\end{definition}
\begin{definition}\label{def:leafset}
$L(\tree)$ is the ordered set of all leaf nodes in $\tree$, starting from the left-most leaf to the right-most leaf.
Given a left and right subtree $\tree'$ and $\tree''$ of the tree $\tree$,
\begin{align*}
    L(\tree) = [L(\tree'); L(\tree'')]
\end{align*}
\end{definition}

\begin{definition}
\textbf{Left-most leaf}  is $L_1(\tree)$ and the  \textbf{right-most leaf}  is $L_{|L(\tree)|}(\tree)$
\end{definition}
\begin{definition}\label{def:succleaf}
\textbf{Successive leaf transitions} are pairs of vertices $(\vrtx_i, \vrtx_j)$, 
\begin{align*}
\begin{split}
    \succleave(\paramtree) = \smashoperator{\bigcup_{\tree \in \treeset(\paramtree)}} \left\{   (L_\trgstep(\tree),  L_{\trgstep+1}(\tree))  ~:~   1 \leq \trgstep < |L(\tree)| \right\}
\end{split}
\end{align*}
where $L_n(\tree)$ is the $n$-th leaf of $\tree$
\end{definition}

\begin{definition}
A \textbf{left boundary} $B_l(\tree)$ of a tree is the set of vertices induced by recursively visiting the left vertex from the $\rootv$.
$$B_l(\tree) = \left\{ \vrtx ~:~ \vrtx = \leftc^k(\rootv), k > 1 \right\} \cup \{\rootv\}$$
The notion is similarly defined for the \textbf{right boundary} $B_r$.
\end{definition}
\begin{definition}
The probability $p(\tree) = \pi(\rootv)$, where $\pi$ is defined recursively as:
\begin{align*}
    \recurse(\vrtx_i) = 
    \left\{\begin{array}{ll}
    \leafp_i      & \textrm{if $\vrtx_i \in \alignment(\tree)$,} \\
    & \\
    (1 - \leafp_i)~\cdot  &  \\
    \quad \recurse(\leftc(\vrtx_i))~\cdot &   \textrm{else}\\
    \qquad \recurse(\rightc(\vrtx_i)) & 
    \end{array}\right.
\end{align*}
where $\leftc(\vrtx_i)$ and $\rightc(\vrtx_i)$ are the left child and right child respectively.
\end{definition}
\begin{proposition}
If $\tree'$ and $\tree''$ are the left and right subtrees of $\tree$ respectively, and $\paramtree'$ and $\paramtree''$ are subtrees of $\paramtree$,
then 
$$\tree \in \treeset(\paramtree) \rightarrow \tree' \in \treeset(\paramtree'), \tree'' \in \treeset(\paramtree'')$$
\end{proposition}
\begin{proof}
\begin{align*}
\rootv(\paramtree) &= \rootv(\tree) \\
\leftc(\rootv(\tree)) &= \rootv(\tree') \\
                    &= \leftc(\rootv(\paramtree))  = \rootv(\paramtree')
\end{align*}
Since the vertices of $\tree'$ and $\tree''$ are subsets of vertices of $\paramtree'$ and $\paramtree''$  respectively, they are each internal trees of $\paramtree'$ and $\paramtree''$.
Therefore $\tree' \in \treeset(\paramtree'), \tree'' \in \treeset(\paramtree'')$
\end{proof}
\begin{proposition}
If for all $v_i \in L(\paramtree) \rightarrow \leafp_i = 1$, then $$\sum_{\tree \in \treeset(\paramtree)} p(\tree) = 1$$ 
\end{proposition}
\begin{proof}
Base case: $\paramtree$ is is of depth 0, then $\treeset(\paramtree) = \{\tree\}$, where $\tree = \paramtree = \rootv$., and since $\rootv$ is a leaf $\leafp = 1$.

Inductive case: Let the left and right subtrees of $\paramtree$ be $\paramtree'$ and $\paramtree''$ respectively, and assume $\sum_{\tree \in \treeset(\paramtree')} p(\tree) = 1$, and same for $\paramtree''$
\begin{align*}
\MoveEqLeft
 \sum_{\tree \in \treeset(\paramtree)} p(\tree) \\
    &= \leafp_{\rootv} +  \sum_{\tree \in (\treeset(\paramtree) \setminus \{\rootv\}} p(\tree) \\
    \intertext{Second term has common factor, since $\rootv$ is not a leaf, }
    &= \leafp_{\rootv} + (1 - \leafp_{\rootv})\smashoperator{\sum_{\substack{\tree' \in \treeset(\paramtree')\\ \tree'' \in \treeset(\paramtree'')}}} \pi(\rootv(\tree')) \cdot \pi(\rootv(\tree'')) \\
    &=     \leafp_{\rootv} + (1 - \leafp_{\rootv}) \smashoperator{\sum_{\substack{\tree' \in \treeset(\paramtree')\\ \tree'' \in \treeset(\paramtree'')}}} p(\tree') \cdot p(\tree'') \\
    &= \leafp_{\rootv} + (1 - \leafp_{\rootv}) 
    \left(\smashoperator[r]{\sum_{\tree' \in \treeset(\paramtree')}}p(\tree') \right)
    \left(\smashoperator[r]{\sum_{\tree'' \in \treeset(\paramtree'')}}  p(\tree'') \right)
    \intertext{By the inductive assumption,}
    &= \leafp_{\rootv} + (1 - \leafp_{\rootv})  \cdot 1 \cdot 1 \\
    &= 1
\end{align*}
\end{proof}
\begin{proposition}
Let
\begin{align*}
    \m(v_i) &= \left(\negm(\parent(v_i))\right)^{\frac{1}{2}} \cdot l_i \\
    \negm(v_i) &= \left(\negm(\parent(\vrtx_i))\right)^{\frac{1}{2}} \cdot (1 - l_i)
\end{align*}
then,
\begin{align*}
    p(\tree) = \prod_{\trgstep=1}^\trglength \m(\alignment_\trgstep(\tree))
\end{align*}
\end{proposition}
\begin{proof}
We can write,
\begin{align}
\prod_{\trgstep=1}^\trglength \m(\alignment_\trgstep(\tree)) =
\prod_{\vrtx \in V^\trglength}\left(\negm(\parent(\vrtx))\right)^{\frac{1}{2}} \cdot \pi(\vrtx) \label{eqn:mprod}
\end{align}
where $V^\trglength = \alignment(\tree)$, and $|V^\trglength| = \trglength$.

If $V^1$, then $V^1 = \{\rootv(\tree)\}$, then $\m(\rootv(\tree)) = \leafp_{\rootv(\tree)}$.

If $|V^\trglength| > 1$, since $\tree$ is a full binary tree, then there exists at least two vertices $\vrtx_i, \vrtx_j \in V$ such that $\parent(\vrtx_i) = \parent(\vrtx_j) = \vrtx_k$.
Let $V^{\trglength - 1} = (V \setminus \{\vrtx_i, \vrtx_j\}) \cup \{\vrtx_k\}$. Then,
\begin{align*}
\MoveEqLeft
\prod_{v \in V}\left(\negm(\parent(\vrtx))\right)^{\frac{1}{2}} \cdot \pi(\vrtx) \\
    &=  \left(\negm(\parent(\vrtx_k))\right)^{\frac{1}{2}} \cdot (1 - l_k) \pi(\vrtx_i) \pi(\vrtx_j) \\
    &\qquad\qquad\smashoperator{\prod_{\vrtx \in (V \setminus \{\vrtx_i, \vrtx_j\})}}
    \left(\negm(\parent(\vrtx))\right)^{\frac{1}{2}} \cdot \pi(\vrtx) \\
    &=   \left(\negm(\parent(\vrtx_k))\right)^{\frac{1}{2}} \cdot \pi(\vrtx_k)\\
    &\qquad\qquad\smashoperator{\prod_{v \in (V \setminus \{\vrtx_i, \vrtx_j\})}}
    \left(\negm(\parent(\vrtx))\right)^{\frac{1}{2}} \cdot \pi(\vrtx) \\
    & = \smashoperator{\prod_{v \in V^{\trglength - 1}}}
    \left(\negm(\parent(\vrtx))\right)^{\frac{1}{2}} \cdot \pi(\vrtx)
\end{align*}
Then $V^{\trglength-1}$ forms another full binary tree $\tree'$, where $\vrtx_k$ is now a leaf, and we can assign $\leafp_k \vcentcolon= \pi(v_k)$
Applying this identity, we can repeatedly reduce the number of factors by 1, until we get $V^1$
\end{proof}

\begin{proposition}\label{prop:leftbound}
If $\tree$ is an internal tree of $\paramtree$,
$$L_1(\tree) \in B_l(\paramtree), L_{|L(\tree)|}(\tree) \in B_r(\paramtree)$$
\end{proposition}
\begin{proof}

If $\tree = \rootv$, then the leftmost vertex is  $\rootv$, which is in $B_l$ by definition.

Otherwise, from Definitions \ref{def:internal} \&  \ref{def:full} we know that if $\leftc(\vrtx)$ for a given $\vrtx$ is $\phi$, then $\vrtx$ is a leaf.
We can then find the left-most leaf of $\tree$ by recursively calling $v = \leftc(\vrtx)$, until $\leftc(\vrtx) = \phi$.
Since all vertices of $\tree$ are vertices of $\paramtree$, and both trees share $\rootv$, the left-most leaf of $\tree$, $\vrtx  \in B_l$
\end{proof}
The argument for the rightmost vertex is symmetric.

\begin{proposition}
Let $\paramtree'$ and $\paramtree''$ be left and right subtrees of $\paramtree$.
Then,
\begin{align*}
    \succleave(\paramtree) = \succleave(\paramtree') \cup \succleave(\paramtree'') \cup (B_l(\paramtree') \times B_r(\paramtree''))
\end{align*}
\end{proposition}
\begin{proof}
$\paramtree$ is a complete tree so the left and right subtree $\paramtree'$ and $\paramtree''$ are both complete trees.
For any $\tree \in \treeset(\paramtree)$, then by Definition \ref{def:leafset}, we can find $\tree'$ and $\tree''$ which are internal trees of $\paramtree'$ and $\paramtree''$ respectively, such that $L(\tree) = [L(\tree'); L(\tree'')]$.
Then,
\begin{align*}
\intertext{For $1 \leq \trgstep < |L(\tree')|$,} 
\MoveEqLeft
 (L_n(\tree), L_{n+1}(\tree))  \\
 &= (L_n(\tree'), L_{n+1}(\tree')) \in \succleave(\paramtree')
 \intertext{For $ |L(\tree')| + 1\leq \trgstep < |L(\tree)|$,} 
\MoveEqLeft
 (L_n(\tree), L_{n+1}(\tree)) \\
 &=  (L_{n - |L(\tree')|}(\tree''), L_{n - |L(\tree')| + 1}(\tree''))  \in \succleave(\paramtree'')
\end{align*}
by Definition \ref{def:succleaf}.

For $n = |L(\tree')|$,  we know from Prop. \ref{prop:leftbound},
\begin{align*}
    L_n(\tree) &= L_n(\tree') \in B_r(\paramtree') \, \\
    L_{n+1}(\tree) &= L_1(\tree'') \in B_l(\paramtree'') 
\end{align*}
Therefore, $$(L_n(\tree), L_{n + 1}(\tree)) \in B_l(\paramtree') \times B_r(\paramtree'')$$
\end{proof}

\newpage
\section{Successive Leaf Construction Algorithm}\label{alg:succ_leaves}
\begin{algorithm}[h]
\begin{algorithmic}
   \caption{\textsc{SuccessiveLeaves}}
   \STATE {\bfseries Input:} vertex $\vrtx_i$
   \STATE {\bfseries Output:} successive leaf transitions $\succleave = \{(v_j, v_k), \dots\}$
   \STATE {\bfseries Output:} left boundary $B_l = \{i, \dots\}$
   \STATE {\bfseries Output:} right boundary $B_r = \{i, \dots\}$
   \IF{$v_i$ is a leaf}
    \STATE $\succleave \leftarrow \{\}$
    \STATE $B_l, B_r \leftarrow \{\vrtx_i\}, \{\vrtx_i\}$
   \ELSE
     \STATE $\succleave', B'_l, B'_r \leftarrow \textsc{SuccessiveLeaves}(\leftc(\vrtx_i))$
     \STATE $\succleave'', B''_l, B''_r \leftarrow \textsc{SuccessiveLeaves}(\rightc(\vrtx_i))$
     \STATE $\succleave \leftarrow \succleave' \cup \succleave'' \cup  (B'_r \times B''_l)$
     \STATE $B_l \leftarrow B'_l \cup \{\vrtx_i\}$
     \STATE $B_r \leftarrow B''_r \cup \{\vrtx_i  \}$   
   \ENDIF
\end{algorithmic}
\end{algorithm}
\newpage
\section{Decoding Algorithm} \label{sec:decode}
\begin{algorithm}[h]
   \caption{\textsc{DecodeJoint}}
   \label{alg:decodejoint}
\begin{algorithmic}
   \STATE {\bfseries Input: $[p(\mathbf{\trgtok}|\vrtx_1), \dots p(\mathbf{\trgtok}|\vrtx_{|V|})]$} 
   \STATE {\bfseries Output:} $\trgtok^*_{1:\trglength^*}$

   \FORALL{$v_i \in V$}
    \STATE $m_{\mathrm{arg}}^*(\vrtx_i) \leftarrow \argmax_\trgtok
    p(\mathbf{\trgtok}=\trgtok|\vrtx_i)$     \quad\COMMENT{Initialise} 
        \STATE $m^*(i) \leftarrow \max_\trgtok
    p(\mathbf{\trgtok}=\trgtok|\vrtx_i)$
   \ENDFOR
   \STATE $\trgstep \leftarrow 1$
   \FORALL{$\vrtx_i \in B_l$}
    \STATE $\M^*(\vrtx_i, 1) \leftarrow m^*(\vrtx_i)$
   \ENDFOR
   \WHILE{$\max_{\vrtx \in V} \M^*(\vrtx, \trgstep) \geq p^*$}
   \IF{$\max_{\vrtx_i \in B_r} \M^*(\vrtx_i, \trgstep) > p^*$}
    \STATE $\trglength^* \leftarrow t$     \qquad \COMMENT{Compute current best}
    \STATE $v^* \leftarrow \argmax_{\vrtx_i \in B_r} \M^*(\vrtx_i, \trglength^*)$
    \STATE $\trgtok^* \leftarrow [m^*_{\mathrm{arg}}(\vrtx^*, \trglength^*)]$
    \ENDIF
    \STATE $t \leftarrow t + 1$
    \FORALL{$v_i \in V$}
        \STATE $\M^*(\vrtx_i, \trgstep) \leftarrow m^*(\vrtx_i) \cdot \displaystyle \max_{\vrtx_j | (\vrtx_j, \vrtx_i) \in \succleave} \M^*(\vrtx_j, \trgstep-1)$
        \STATE $\M_{\mathrm{arg}}^*(\vrtx_i, \trgstep) \leftarrow  \displaystyle \argmax_{\vrtx_j | (\vrtx_j, v_i) \in \succleave} \M^*(\vrtx_j, \trgstep-1)$
    \ENDFOR
    
   \ENDWHILE
    \FOR{$t \leftarrow \trglength^*~\mathrm{to}~2$}
        \STATE $v^* \leftarrow \M_{\mathrm{arg}}^*(v^*, t)$ \qquad \COMMENT{Backtrace}
        \STATE $\trgtok^* \leftarrow [m^*_{\mathrm{arg}}(v^*, t)].\trgtok^* $
    \ENDFOR
\end{algorithmic}
\end{algorithm}

\section{Experiments}
\subsection{Encoder}\label{sec:encoder}

Before the embeddings are fed into the OM, we first produce contextualised embeddings, by first feeding it into a one layer bidirectional Gated Recurrent Unit (GRU; \citealt{cho2014properties}).
We then expose the following representations from the encoder to the decoder:

$\encode_\rho$ --- Final representation computed by OM. Can be thought of as the root representation.

$\encode_\iota$ --- Intermediate states ($\hat{M}_1 \dots \hat{M}_\srclength$) concatenated. Can be thought of as the representations of the internal nodes \emph{and} the leaves.

$\encode_\ell$ --- Input representations to the OM. Can be thought of as the representation of the  leaves.

$\encode_{ce}$ --- Contextualized embeddings from the GRU.

$\encode_e$ --- Embeddings fed to the GRU.

We also use the $\cell(\cdot, \cdot)$ function as defined in the paper.

\newpage
\subsection{SCAN sample trees}

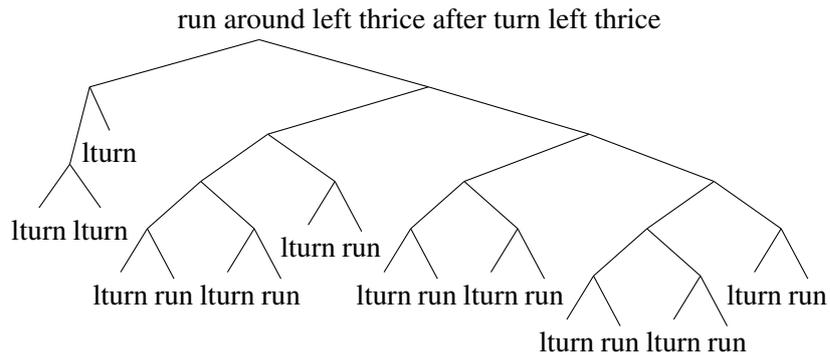
\begin{figure}[ht]
\scalebox{1}{
\begin{minipage}{1.0\linewidth}
\begin{center}
\begin{tabular}{c}
run around left thrice after turn left thrice\\
\begin{minipage}{1.0\textwidth}
\begin{center}
\begin{forest}
shape=coordinate,
deeper
[[[[lturn] [lturn]] [lturn]] [[[[[lturn] [run]] [[lturn] [run]]] [[lturn] [run]]] [[[[lturn] [run]] [[lturn] [run]]] [[[[lturn] [run]] [[lturn] [run]]] [[lturn] [run]]]]]]
\end{forest}
\end{center}
\end{minipage}
\end{tabular}
\end{center}
\end{minipage}}
\caption{Erroneous tree example from the model trained on the \textsc{Length} split.}
\label{fig:scantreesneg}
\end{figure}
\subsection{Multi30k Translation Sample Trees}
\begin{figure}[ht]
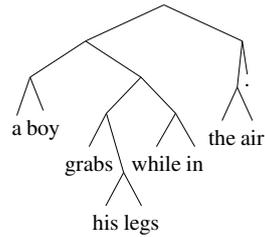

\begin{center}
\scalebox{0.85}{\treetablecell{Ein älterer Mann spielt ein Videospiel.}{[[[[an], [[older] [man]]], [[is] [playing]]], [[[a] [[video] [game]]] [.]]]}} \\
\scalebox{0.85}{\treetablecell{Ein Mädchen an einer Küste mit einem Berg im Hintergrund.}{[[[[a] [girl]], [[[on] [[a] [shore]]] [with]]] [[[a] [mountain]] [[[in] [[the] [background]]] [.]]]]}}\\
\scalebox{0.85}{\treetablecell{Ein Junge greift sich ans Bein während er in die Luft springt.}{[[[[a] [boy]], [[[grabs] [[his] [legs]]], [[while] [in]]]], [[[the] [air]] [.]]]}}
\caption{Trees found by our model from Multi30K De-En.}
\end{center}
\end{figure}
\end{document}